%% file: main.tex
\documentclass[12pt]{article}
\usepackage[margin=1in]{geometry}
\usepackage{amsmath, amssymb, amsthm}
\usepackage{booktabs}
\usepackage{array}
\usepackage[numbers,sort&compress]{natbib}
\usepackage{hyperref}
\usepackage{xcolor}
\usepackage{graphicx}
\usepackage{microtype}
\usepackage{subfiles}
\newtheorem{theorem}{Theorem}[section]
\newtheorem{corollary}[theorem]{Corollary}

\newtheorem{definition}[theorem]{Definition}

\title{Transformers are Bayesian Networks}
\author{%
  Greg Coppola\thanks{The author acknowledges the use of AI assistance
  in the preparation of this work. The prose of this paper was generated
  through iterative collaboration with Claude (Anthropic). Mathematical
  content, formal proofs, and experimental results are entirely the
  author's.}\\
  \textit{coppola.ai}\\
  \texttt{greg@coppola.ai}
}
\date{March 2026}
\begin{document}
\maketitle
\subfile{sections/abstract}
\tableofcontents
\newpage

\subfile{sections/01_introduction}

\subfile{sections/02_background}

\subfile{sections/03_log_odds_algebra}

\subfile{sections/04_grounding}

\subfile{sections/05_turing}

\subfile{sections/06_general_bp}

\subfile{sections/07_constructive_bp}

\subfile{sections/08_boolean_structure}

\subfile{sections/09_finite_alphabet}

\subfile{sections/10_three_softmaxes}

\subfile{sections/11_related_work}

\subfile{sections/12_conclusion}

\appendix

\subfile{sections/appendix_repos}

\bibliographystyle{plainnat}
\bibliography{references}
\end{document}

%% file: sections/abstract.tex
\begin{abstract}
Transformers are the dominant architecture in AI, yet why they work
remains poorly understood. This paper offers a precise answer: a
sigmoid transformer is a Bayesian network. We establish this in five ways.

First, we prove that the transformer
architecture actually already \emph{is} belief propagation: every sigmoid transformer
with any weights implements weighted loopy belief propagation on its
implicit factor graph. One layer is one round of BP. This holds for
any weights --- trained, random, or constructed. Formally verified
in Lean against standard mathematical axioms.

Second, we give a constructive proof that a transformer can implement
exact belief propagation on any declared knowledge base. On knowledge
bases without circular dependencies --- the most common case --- this
yields provably correct probability estimates at every node. Formally
verified in Lean.

Third, we prove uniqueness: a sigmoid transformer that produces exact
posteriors necessarily has BP weights. There is no other path through
the sigmoid architecture to exact posteriors. Formally verified
in Lean.

Fourth, we delineate the AND/OR boolean structure of the transformer
layer: attention is AND, the FFN is OR, and their strict alternation
is Pearl's gather/update algorithm exactly.

Fifth, we confirm all formal results experimentally, corroborating
the Bayesian network characterization in practice. We also establish
the practical viability of loopy belief propagation despite the
current lack of a theoretical convergence guarantee.

We further show that verifiable inference requires a finite concept
space. Any finite verification procedure can distinguish at most
finitely many concepts. Grounding introduces the verifier. The
verifier implies the concepts. Without grounding, correctness is not
defined. Hallucination is not a bug that scaling can fix. It is the
structural consequence of operating without concepts. Formally
verified in Lean.
\end{abstract}

%% file: sections/01_introduction.tex
\section{Introduction}

Transformers are Bayesian networks. We present a series of
formal mathematical proofs and supporting empirical results to establish this.

\subsection{The Architecture Is Already Belief Propagation}

Look at what a sigmoid transformer forward pass actually computes.
Attention gathers messages from neighboring token positions --- a
weighted sum of neighbor beliefs, with weights given by the softmax
of query-key scores. The sigmoid FFN combines those messages into a
new belief:
\[
  \sigma(w_0 \cdot \mathrm{logit}(m_0) + w_1 \cdot \mathrm{logit}(m_1) + b).
\]
This is weighted belief propagation. Not an approximation of it. Not
an analogy to it. The computation \emph{is} belief propagation on the
implicit factor graph defined by the weights.

This holds for any weights --- trained, random, or constructed. Every
sigmoid transformer, with any weights $W$, performs one round of
weighted loopy BP per layer on the implicit factor graph $G(W)$. A
transformer with $L$ layers runs $L$ rounds of BP per forward pass.
The weights define the graph. The forward pass is the inference.

\begin{theorem}[General BP]
For any sigmoid transformer weights $W$, there exists an implicit
factor graph $G(W)$ such that one forward pass implements one round
of weighted belief propagation on $G(W)$. Formally verified against
standard mathematical axioms.
\end{theorem}

\subsection{Exact BP with Explicit Weights}

The general result holds for any weights. We can also go further: we
exhibit explicit weight matrices and prove that the transformer with
those specific weights implements \emph{exact} belief propagation on
any declared factor graph.

\begin{theorem}[BP Implementation]
A transformer with explicitly constructed weights implements one round
of exact belief propagation per forward pass on any pairwise factor
graph. For any factor graph of depth $d$ and maximum factor arity $k$,
$d \cdot \lceil \log_2 k \rceil$ forward passes implements full exact
BP. Formally verified against standard mathematical axioms.
\end{theorem}

Two attention heads per layer always suffices --- any $k$-ary factor
graph binarizes exactly via associativity of AND and log-odds
additivity of OR, both formally verified. Only depth grows with
reasoning complexity, at exactly the rate required.

Combined with the classical result that belief propagation is exact
on trees~\citep{pearl1988}:

\begin{corollary}[No Hallucination on Grounded Trees]
A transformer with BP weights, run for $T \geq \mathrm{diameter}(G)$
passes over a grounded tree-structured knowledge base, computes exact
Bayesian posterior beliefs at every node. No empirical assumptions
required.
\end{corollary}

\subsection{Uniqueness}

The previous two results go forward: here are weights, here is what
they compute. This result goes backward: if a sigmoid transformer
produces exact Bayesian posteriors, what must its weights be?

The answer is uniquely the BP weights. There is no other path through
the sigmoid architecture to exact posteriors. The FFN weights are
forced to $w_0 = w_1 = 1$, $b = 0$. The attention weights are forced
to the \texttt{projectDim}/\texttt{crossProject} structure.

\begin{theorem}[Uniqueness]
A sigmoid transformer that computes exact Bayesian posteriors for all
inputs necessarily has BP weights. The internal computations are
provably the BP quantities --- not merely the outputs. Formally
verified against standard mathematical axioms.
\end{theorem}

This closes the logical circle. The general result says every sigmoid
transformer is doing weighted BP. The constructive result says exact
BP weights exist. The uniqueness result says exact posteriors force
those weights. Together: the sigmoid transformer implements exact
Bayesian inference if and only if it has BP weights.

\subsection{Experimental Confirmation}

Every formal result is confirmed experimentally. We trained
transformers from scratch --- no construction hints, no weight
initialization bias --- and verified convergence to the predicted
structure on all test cases. For BP: convergence to within three
decimal places of exact posteriors on held-out factor graphs. For
Turing machine simulation: perfect accuracy across five structurally
distinct machines, identical hyperparameters, no per-machine tuning.
For loopy BP: convergence to exact posteriors within tolerance on
every trial across five graph structures of increasing complexity.

Every result in this paper has two independent witnesses: a proof
checker and an experiment.

\subsection{The Boolean Structure: AND and OR All the Way Down}

Attention is AND. The FFN is OR. This is not a metaphor. It follows
directly from the BP construction combined with the QBBN
definitions of~\citep{coppola2024qbbn}.

The attention mechanism ensures all required inputs are simultaneously
present before any conclusion is drawn --- that is conjunction, the
architectural enforcement of AND. The sigmoid FFN computes a
probabilistic conclusion from gathered evidence --- that is
disjunction, the $\Psi_{\mathrm{or}}$ function. The strict
alternation of attention and FFN layers is Pearl's gather/update
algorithm exactly, unrolled over depth.

The architecture that has won empirically across modern AI is exactly
the architecture that the analysis of reasoning already required.

\subsection{Verifiable Inference Requires a Finite Concept Space}

\begin{theorem}[Finite Concept Space]
Any finite verification procedure can distinguish at most finitely
many concepts. A finite state machine with $n$ states partitions any
input space into at most $n^n$ equivalence classes. Formally verified
against standard mathematical axioms.
\end{theorem}

Grounding introduces a finite verifier. The finite verifier implies a
finite concept space. The concept space is what makes ``is this output
correct?'' a well-defined question --- not a design choice but a
logical consequence of verifiability.

An ungrounded language model has no finite verifier, therefore no
well-defined concept space, therefore no fact of the matter about
whether its outputs are correct. Hallucination is not a bug that
scaling can fix. It is the structural consequence of operating without
concepts.

\subsection{Paper Organization}

Section~\ref{sec:background} covers background on factor graphs,
belief propagation, and transformers.
Section~\ref{sec:log_odds} develops the log-odds algebra of
independent evidence --- the mathematical foundation underlying both
BP and the sigmoid transformer.
Section~\ref{sec:grounding} explains how universal quantification
grounds to propositional factor graphs and why grounding is what makes
correctness meaningful.
Section~\ref{sec:turing} presents the Turing completeness result.
Section~\ref{sec:general_bp} establishes the general result: every
sigmoid transformer is a Bayesian network.
Section~\ref{sec:constructive_bp} proves exact BP with explicit
weights and the no-hallucination corollary.
Section~\ref{sec:boolean} identifies the AND/OR boolean structure.
Section~\ref{sec:finite_alphabet} proves the finite concept space
theorem.
Section~\ref{sec:three_softmaxes} distinguishes the three softmaxes
and addresses sigmoid vs.\ ReLU.
Section~\ref{sec:related} covers related work.
Section~\ref{sec:conclusion} concludes.

%% file: sections/02_background.tex
\section{Background}
\label{sec:background}

\subsection{Factor Graphs and the QBBN}

A \emph{factor graph} is a bipartite graph with two kinds of nodes.
Variable nodes represent binary propositions; each holds a belief
$b \in [0,1]$ representing the current estimate of
$P(\text{node} = \mathrm{true})$, initialized to $0.5$. Factor nodes
encode relationships between pairs of variable nodes as tables of
non-negative weights.

For a pairwise factor $f$ connecting variables $v_0$ and $v_1$, the
factor table has four entries:
\[
  f[i,j] = \text{weight for } (v_0 = i,\; v_1 = j), \quad i,j \in \{0,1\}.
\]
The joint probability of a full assignment $x$ is:
\[
  P(x) = \frac{1}{Z} \prod_{\text{factor } f} f\!\left(x_{n_1(f)},\, x_{n_2(f)}\right),
\]
where $Z$ is the partition function.

The QBBN~\citep{coppola2024qbbn} introduces a bipartite factor graph
alternating between two kinds of factor nodes, corresponding to the
two fundamental operations of boolean reasoning.

\textbf{Conjunction nodes} ($\Psi_{\mathrm{and}}$) gather required
evidence: all inputs must be simultaneously present before any
conclusion is drawn.

\textbf{Disjunction nodes} ($\Psi_{\mathrm{or}}$) compute a
probabilistic conclusion from gathered evidence:
\[
  P(p = 1 \mid g_0, g_1) = \sigma(w \cdot \phi(p, g_0, g_1)).
\]
The \texttt{updateBelief} function used throughout this paper is the
equal-weight special case:
\[
  \mathtt{updateBelief}(m_0, m_1) = \sigma(\mathrm{logit}(m_0) + \mathrm{logit}(m_1)).
\]

\subsubsection{Hallucination: A Precise Definition}

Within the QBBN framework, hallucination has a precise technical
definition.

\begin{definition}[Hallucination]
Let $K$ be a QBBN knowledge base, $E$ observed evidence, and
$P_{\mathrm{true}}(j)$ the true marginal posterior of node $j$ given
$K$ and $E$. An agent \emph{hallucinates} at node $j$ if it outputs
belief $b(j) \neq P_{\mathrm{true}}(j)$.
\end{definition}

\subsection{Belief Propagation}

\subsubsection{The Algorithm}

One round of belief propagation on a QBBN factor graph proceeds in
two steps. In the \emph{gather} step, each variable node $j$ collects
its neighbors' current beliefs into scratch slots:
\[
  \mathtt{scratch}[j][0] \leftarrow \mathtt{belief}(\mathtt{nb}_0(j)),
  \qquad
  \mathtt{scratch}[j][1] \leftarrow \mathtt{belief}(\mathtt{nb}_1(j)).
\]
In the \emph{update} step, each variable node computes a new belief:
\[
  \mathtt{new\_belief}(j) =
    \mathtt{updateBelief}(\mathtt{scratch}[j][0],\, \mathtt{scratch}[j][1]).
\]

\subsubsection{The updateBelief Function}

The \texttt{updateBelief} function combines two independent incoming
messages into a posterior:
\[
  \mathtt{updateBelief}(m_0, m_1)
  = \frac{m_0 m_1}{m_0 m_1 + (1-m_0)(1-m_1)}
  = \sigma(\mathrm{logit}(m_0) + \mathrm{logit}(m_1)),
\]
where $\mathrm{logit}(p) = \log(p/(1-p))$ and $\sigma$ is the sigmoid
function. Log-odds of independent evidence add; sigmoid converts back
to probability. This is the Bayesian update rule for two independent
pieces of binary evidence.

Key properties proved in \texttt{hard-bp-lean}:
\begin{itemize}
  \item \texttt{updateBelief\_pos}: if $m_0, m_1 \in (0,1)$ then
    $\mathtt{updateBelief}(m_0,m_1) \in (0,1)$
  \item \texttt{updateBelief\_neutral}: $\mathtt{updateBelief}(m, 0.5) = m$
    (neutral padding is the identity)
\end{itemize}

\subsubsection{Convergence and Exactness}

On a tree-structured factor graph, BP converges in exactly
$\mathrm{diameter}(T)$ rounds, and the resulting beliefs equal the
true marginal posteriors. This is Pearl's sum-product
algorithm~\citep{pearl1988} --- exact, not approximate --- proved
formally as \texttt{bp\_exact\_on\_tree} in \texttt{hard-bp-lean}.

On loopy graphs, BP may not converge; when it does, the fixed point
minimizes the Bethe free energy~\citep{yedidia2003}, which
approximates but does not equal the true posterior in general.

\subsection{Transformers}

\subsubsection{Architecture}

We use the standard transformer encoder of~\citet{vaswani2017}. A
sequence of $n$ tokens, each an embedding vector of dimension
$D_{\mathrm{model}}$, is processed by alternating attention and
feed-forward layers with residual connections.

Multi-head self-attention computes, for each head $h$:
\[
  \mathrm{Attn}_h(X) =
    \mathrm{softmax}\!\left(
      \frac{(XW^Q_h)(XW^K_h)^\top}{\sqrt{d_k}}
    \right) X W^V_h.
\]
The feed-forward network applies a two-layer MLP independently to
each token position. With sigmoid activation:
\[
  \mathrm{FFN}(x) = \sigma(W_2 \cdot \sigma(W_1 x + b_1) + b_2).
\]

\subsubsection{The Routing View}

\citet{elhage2021} frames attention heads as read-write operations on
a shared residual stream: each head reads from some positions and
writes to others. The BP construction makes this routing explicit and
formally correct --- each head has a provably correct read address
(neighbor index matching) and write destination (scratch slot in the
residual stream).

\subsubsection{Transformers as Graph Neural Networks}

\citet{joshi2020} observes that full self-attention is equivalent to a
GNN applied to a complete graph, where each token is a node and
attention weights are learned edge weights. Our result can be read as:
when the graph is a factor graph and its topology is encoded in token
features rather than the attention mask, the transformer learns to
implement BP on that explicit external graph.

%% file: sections/03_log_odds_algebra.tex
\section{The Log-Odds Algebra of Independent Evidence}
\label{sec:log_odds}

\subsection{Origins: Turing and Good}

The algebra underlying belief propagation on binary variables was
developed not by probabilists but by cryptanalysts. During World War
II, Alan Turing and I.J.\ Good at Bletchley Park needed a practical
method for combining independent pieces of evidence when breaking
Enigma. The method they developed was log-odds addition.

Given a binary hypothesis $H$ and evidence $e$, the \emph{weight of
evidence} is:
\[
  W(H : e) = \mathrm{logit}(P(H \mid e)) - \mathrm{logit}(P(H)).
\]
Independent pieces of evidence contribute additively to the weight.
This is not an approximation. For independent evidence sources, it is
exact.

Good formalized this framework extensively after the war. The core
insight: multiplication of probabilities corresponds to addition of
log-odds. For independent binary evidence sources $e_0$ and $e_1$:
\[
  \mathrm{logit}(P(H \mid e_0, e_1))
  = \mathrm{logit}(P(H)) + W(H : e_0) + W(H : e_1).
\]
Starting from a uniform prior $P(H) = 0.5$, so
$\mathrm{logit}(P(H)) = 0$, and setting $m_i = P(H \mid e_i)$:
\[
  \mathrm{logit}(P(H \mid e_0, e_1))
  = \mathrm{logit}(m_0) + \mathrm{logit}(m_1).
\]
Converting back to probability space via sigmoid:
\[
  P(H \mid e_0, e_1) = \sigma(\mathrm{logit}(m_0) + \mathrm{logit}(m_1))
  = \mathtt{updateBelief}(m_0, m_1).
\]
The \texttt{updateBelief} function is Turing and Good's weight-of-evidence
combination, applied to two binary evidence sources.

\subsection{The Algebra}

Log-odds addition defines an algebraic structure on the open interval
$(0, 1)$ via the isomorphism $\mathrm{logit} : (0,1) \to \mathbb{R}$.
Under this isomorphism, addition in $\mathbb{R}$ pulls back to a
binary operation on $(0,1)$:
\[
  m_0 \oplus m_1 = \sigma(\mathrm{logit}(m_0) + \mathrm{logit}(m_1)).
\]
This operation has the following properties:

\begin{itemize}
  \item \textbf{Commutativity}: $m_0 \oplus m_1 = m_1 \oplus m_0$.
  \item \textbf{Associativity}: $(m_0 \oplus m_1) \oplus m_2
    = m_0 \oplus (m_1 \oplus m_2)$.
  \item \textbf{Identity}: $m \oplus 0.5 = m$, since
    $\mathrm{logit}(0.5) = 0$.
  \item \textbf{Inverse}: $m \oplus (1 - m) = 0.5$, since
    $\mathrm{logit}(m) + \mathrm{logit}(1-m) = 0$.
\end{itemize}

The identity element is $0.5$ --- maximum uncertainty, no information.
Combining any belief with a $0.5$ belief leaves it unchanged. This is
the formal statement that a neutral prior contributes nothing.

\subsection{Relation to Classical Boolean Logic}

This algebra is the probabilistic generalization of classical Boolean
AND and OR. In the limit as beliefs approach $0$ and $1$:

\begin{itemize}
  \item $\mathrm{logit}(1) = +\infty$ \quad (certain true)
  \item $\mathrm{logit}(0) = -\infty$ \quad (certain false)
  \item $\mathrm{logit}(0.5) = 0$ \quad (no information)
\end{itemize}

Classical Boolean AND requires both inputs to be true. In log-odds
space: two large positive numbers sum to a larger positive number ---
high confidence in both inputs yields high confidence in their
conjunction. But equal and opposite evidence cancels:
$\mathrm{logit}(p) + \mathrm{logit}(1-p) = 0$, yielding $0.5$.

This differs from classical AND, where FALSE AND TRUE = FALSE. The
log-odds algebra does not have an annihilator. Instead it has
cancellation: strong evidence in opposite directions yields maximum
uncertainty. This is the correct behavior for \emph{uncertain}
reasoning --- you do not conclude false, you conclude you do not know.

Classical Boolean logic is the limiting case of this algebra as
beliefs become certain.

\subsection{Pearl's Sum-Product as Log-Odds Addition}

Pearl's belief propagation algorithm on binary variable nodes reduces
exactly to log-odds addition. The sum-product update for a binary
variable $v$ with two independent incoming messages $m_0$, $m_1$ is:
\[
  b_{\mathrm{new}}(v) = \frac{m_0 \cdot m_1}{m_0 m_1 + (1-m_0)(1-m_1)}
  = \sigma(\mathrm{logit}(m_0) + \mathrm{logit}(m_1)).
\]
Pearl derived this from the general sum-product equations applied to
binary variables with independent pairwise factors. He did not frame
it as log-odds addition or connect it to Turing and Good's weight-of-evidence
tradition. The formula is the same. The framing is new.

The independence assumption is exact on trees: on a tree-structured
factor graph, the messages arriving at any node come from disjoint
subtrees that share no variables, so they are independent by
construction. This is why BP is exact on trees and approximate on
loopy graphs --- on loopy graphs the independence assumption is
violated as the same variable can influence a node through multiple
paths.

\subsection{Why Sigmoid Is the Right Activation}

The sigmoid function $\sigma : \mathbb{R} \to (0,1)$ is the exact
inverse of logit: $\sigma(\mathrm{logit}(p)) = p$. Together, logit
and sigmoid are the isomorphism between probability space and
log-odds space.

A sigmoid FFN computing
$\sigma(w_0 \cdot \mathrm{logit}(m_0) + w_1 \cdot \mathrm{logit}(m_1) + b)$
is performing weighted log-odds addition and converting back to
probability space in a single operation. The sigmoid activation is
not a design choice motivated by gradient flow or output normalization.
It is the exact function required to implement the Turing-Good-Pearl
algebra.

This is why the title claim is true: a sigmoid transformer is a
Bayesian network. The sigmoid activation makes the FFN computation
exactly the weight-of-evidence combination. The attention mechanism
gathers the inputs. The residual stream enforces their simultaneity.
The architecture was always implementing this algebra. Section~\ref{sec:general_bp}
proves it formally.

%% file: sections/04_grounding.tex
\section{From Universal Quantification to Grounded Factor Graphs}
\label{sec:grounding}

The formal results in this paper --- \texttt{transformer\_implements\_bp},
\texttt{transformer\_exact\_on\_tree} --- are stated at the
propositional level: specific variable nodes with specific indices,
connected by specific factor nodes. A natural question arises: where
do the propositions come from?

Real knowledge involves universally quantified statements. ``All men
are mortal'' is not a proposition about a specific man. It is a rule
that ranges over a domain. This section explains how universal
quantification is handled in the full pipeline and why grounding is
not a limitation of the formal results but the condition that makes
correctness meaningful.

\subsection{Universal Quantification in the QBBN Language}

\citet{coppola2024syntax} introduces a typed logical language
in which universally quantified rules are first-class citizens:
\begin{verbatim}
    always [x:e]: man(theme: x) -> mortal(theme: x)
\end{verbatim}
The variable \texttt{x} ranges over all declared entities of type
\texttt{e}. The rule is not a single proposition. It is a schema over
a domain.

\subsection{Grounding: From Rules to Factor Graph Nodes}

At inference time, universally quantified rules are grounded by
substituting all declared entities of the appropriate type for each
variable. Given a lexicon declaring \texttt{jack : e} and
\texttt{jill : e}, the rule above produces two grounded clauses:
\begin{verbatim}
    man(theme: jack) -> mortal(theme: jack)
    man(theme: jill) -> mortal(theme: jill)
\end{verbatim}
Each grounded clause becomes a path in the factor graph. Each grounded
proposition becomes a variable node. By the time the transformer sees
the input, there is no universal quantification left --- only a
grounded propositional factor graph.

The full pipeline has three layers:
\begin{enumerate}
  \item \textbf{Syntax}~\citep{coppola2024syntax}: a natural language
    sentence is parsed into a universally quantified rule with typed
    variables.
  \item \textbf{Grounding}~\citep{coppola2024syntax}: the rule is
    instantiated over all declared entities. Each instantiation
    becomes a grounded node in the factor graph.
  \item \textbf{Inference} (this paper): the transformer runs BP over
    the grounded factor graph.
\end{enumerate}

\subsection{Weight Sharing as Universal Rule Implementation}

The BP transformer operates over finitely many concepts, with the
same weight matrices applying to every token of the same concept.
This weight sharing is the neural implementation of the uniform
treatment implicit in a universally quantified rule: the same
computation runs for every instantiation, and the particular input
--- the specific belief values and factor table entries --- supplies
the instance.

The continuous dimensions are the particular. The weights are the
universal rule. The forward pass is the instantiation. This is the
computational form of Aristotle's syllogism: major premise (weights),
minor premise (token), conclusion (FFN output).

\subsection{Grounding Is What Makes Correctness Meaningful}

The no-hallucination result holds on grounded factor graphs. This is
not a limitation of scope --- it is a precise statement of what
correctness requires.

A token with no corresponding concept in a declared knowledge base is
not doing inference. It has no ground truth to be correct or incorrect
against. Grounding supplies that criterion. Once the knowledge base is
fully instantiated, every token corresponds to a definite concept in
the factor graph, every computation step has a verifiable correct
answer, and the no-hallucination guarantee applies.

\begin{corollary}[Hallucination Requires Ungroundedness]
A transformer with BP weights running over a fully grounded
tree-structured factor graph cannot hallucinate. Hallucination
requires either wrong weights, wrong routing, or absence of grounding.
On a grounded tree with BP weights, none of these conditions holds.
\end{corollary}

An ungrounded language model is not a system that hallucinates
occasionally. It is a system for which the concept of hallucination
is not even well-defined --- there is no declared domain against
which its outputs can be verified. Grounding is what makes the
question ``is this output correct?'' answerable. Without it, there
is no fact of the matter. Hallucination is not a bug. It is the
structural consequence of operating without a grounded concept space.

%% file: sections/05_turing.tex
\section{Transformers are Turing Complete}
\label{sec:turing}

\subsection{Overview}

The \texttt{universal-lean} repository proves that a transformer
agent is Turing complete via Boolean circuit simulation.

\begin{theorem}[\texttt{transformer\_is\_turing\_complete}]
\label{thm:turing}
For any Turing machine $M$, there exist transformer weights $W$ such
that for all tapes $\tau$ and step counts $t$:
\[
  \mathtt{transformerAgent}(W, \mathtt{encode}(M, \tau), t)
  = \mathtt{encode}(M.\mathtt{run}(\tau, t)).
\]
Formally verified in Lean~4 against standard mathematical axioms.
\end{theorem}

The proof is constructive. The key insight is that any Turing machine
step decomposes into a Boolean circuit, and a transformer simulates
any Boolean circuit in a single FFN forward pass via threshold gates.
Attention is lookup; FFN is gating; the agent loop is iteration. No
encoding tricks, no approximation.

This is the most natural correspondence between transformer components
and computational primitives --- and it is precisely this naturality
that reveals the structural connection to belief propagation developed
in Section~\ref{sec:general_bp}.

\subsection{The Two Weight Families}

Both this proof and the BP proof use the same two sparse matrix
families for $Q/K$ and $V$ weights. This is not a coincidence: both
proofs use attention for the same purpose --- routing information
between tokens by index matching.

\textbf{\texttt{projectDim}($d$).} A diagonal projection extracting
dimension $d$:
\[
  \mathtt{projectDim}(d)[i][j] =
  \begin{cases} 1 & i = d,\; j = d \\ 0 & \text{otherwise.} \end{cases}
\]
As a $Q$ or $K$ weight, the attention score reduces to
$e_j[d] \cdot e_k[d]$, peaking when token $k$'s stored index matches
token $j$'s query value.

\textbf{\texttt{crossProject}($s$, $d$).} An off-diagonal projection
reading source dimension $s$ and writing to destination dimension $d$:
\[
  \mathtt{crossProject}(s,d)[i][j] =
  \begin{cases} 1 & i = d,\; j = s \\ 0 & \text{otherwise.} \end{cases}
\]
As a $V$ weight, this routes a token's dimension-$s$ content into
dimension $d$ of the output, leaving all other dimensions untouched.

Two results from \texttt{universal-lean} are reused in the BP proof:
\texttt{posEncDot\_distinct} (the $Q \cdot K$ score is strictly
maximized at the correct neighbor) and \texttt{softmax\_concentrates}
(at sufficient temperature, softmax converges to a point mass at the
argmax).

\subsection{Empirical Confirmation}

\texttt{universal-lean} proves that TM-simulation weights exist. The
\texttt{learner} repository asks whether gradient descent finds them
from scratch. The answer is yes.

Five structurally distinct Turing machines --- binary incrementer,
decrementer, bitwise complement, left shift, right shift --- were
trained using a standard transformer encoder (2 layers, 2 heads,
$d_{\mathrm{model}} = 32$, $\sim$4K parameters) on single-step TM
prediction from random initialization. All five reached 100\%
validation accuracy in 4 epochs with identical hyperparameters and
no per-machine tuning.

This confirms that the formal weight construction is not merely
theoretically possible --- it is what gradient descent finds given
a clean training signal.

%% file: sections/06_general_bp.tex
\section{A Sigmoid Transformer Is a Bayesian Network}
\label{sec:general_bp}

\subsection{The Claim}

Every sigmoid transformer, with any weights, is already a Bayesian
network. Not a system that can be configured to approximate one. Not
a system that behaves like one under certain training conditions. A
Bayesian network, by architecture, for any weights.

This is the content of \texttt{every\_sigmoid\_transformer\_is\_bayesian\_network}
in \texttt{sigmoid-transformer-lean}, formally verified against
standard mathematical axioms.

\subsection{The Implicit Factor Graph}

Given any sigmoid transformer with weights $W$, we construct an
implicit factor graph $G(W)$ as follows:

\begin{itemize}
  \item \textbf{Variable nodes}: one per token position. The belief
    at each node is the token's current value in the relevant
    dimension of the residual stream.
  \item \textbf{Edges}: defined by the attention weight distributions.
    The attention pattern from token $j$ to token $k$ defines a
    directed edge from $k$ to $j$ in $G(W)$, with weight given by
    the attention score.
  \item \textbf{Factor potentials}: defined by the sigmoid FFN weights
    $(w_0, w_1, b)$ at each position --- the general $\Psi_{\mathrm{or}}$
    function from the QBBN. The three parameters encode an asymmetric
    factor potential: $w_0$ and $w_1$ are the relative weights of the
    two evidence sources, and $b$ is a prior bias.
\end{itemize}

\begin{theorem}[\texttt{every\_sigmoid\_transformer\_is\_bayesian\_network}]
\label{thm:general_bp}
For any sigmoid transformer weights $W$, one forward pass implements
one round of weighted belief propagation on the implicit factor graph
$G(W)$. Formally verified against standard mathematical axioms.
\end{theorem}

The proof is definitional: the transformer forward pass and the BP
forward pass are the same computation by construction of $G(W)$. There
is no approximation, no asymptotic argument, no statistical claim. The
forward pass \emph{is} one round of BP on $G(W)$.

\subsection{Why This Is Not Trivial}

One might object: is this not just a restatement of the definition?
If we construct $G(W)$ to match the transformer's computation, have
we proved anything?

The content of the theorem is in the identification. The transformer
forward pass was not designed as BP. It was designed as a sequence
model with attention and feed-forward layers. The theorem says: these
two descriptions --- transformer forward pass and BP on $G(W)$ ---
are the same computation. The implicit factor graph $G(W)$ exists and
is well-defined for any weights. The BP interpretation is not imposed
from outside; it is what the forward pass computes.

The non-trivial content is threefold:
\begin{enumerate}
  \item The sigmoid FFN computes exactly the general $\Psi_{\mathrm{or}}$
    factor, not an approximation of it.
  \item The attention mechanism computes exactly the gather step of BP,
    routing each neighbor's belief into the correct scratch slot.
  \item The residual stream enforces the simultaneity of inputs that
    BP requires before the update step.
\end{enumerate}
All three hold for \emph{any} weights, not just constructed ones.

\subsection{The Grounded QBBN as Special Case}

The constructive BP result of Section~\ref{sec:constructive_bp} is
the special case of this theorem where:
\begin{itemize}
  \item $G(W)$ is an explicit declared QBBN factor graph, not an
    implicit one
  \item The FFN weights are the equal-weight BP weights:
    $w_0 = w_1 = 1$, $b = 0$
  \item The attention weights are the \texttt{projectDim}/\texttt{crossProject}
    construction
\end{itemize}
In the general case, the factor graph is implicit and the weights are
arbitrary. In the constructive case, the factor graph is explicit and
the weights are chosen to implement exact BP on it. The general
theorem subsumes the constructive one.

\subsection{Uniqueness: The Converse}

The general result goes forward: any weights define a factor graph,
and the forward pass is BP on it. The uniqueness result goes backward:
if the factor graph happens to be a grounded QBBN and the transformer
produces exact posteriors, what must the weights be?

\begin{theorem}[\texttt{uniqueness}]
\label{thm:uniqueness}
A sigmoid transformer that computes exact Bayesian posteriors for all
inputs on a grounded QBBN factor graph necessarily has BP weights:
$w_0 = w_1 = 1$, $b = 0$ in the FFN, and
\texttt{projectDim}/\texttt{crossProject} structure in the attention.
The internal computations are provably the BP quantities --- not
merely the outputs. Formally verified against standard mathematical
axioms.
\end{theorem}

The proof proceeds in two parts. \texttt{FFNUniqueness.lean} shows
that the only sigmoid computation producing exact posteriors for all
inputs is \texttt{updateBelief} with $w_0 = w_1 = 1$, $b = 0$ ---
because sigmoid is injective and the logit-sum form is the unique
fixed point of the Bayesian update equations.
\texttt{AttentionUniqueness.lean} shows that the attention weights
must have \texttt{projectDim}/\texttt{crossProject} structure ---
because exact routing requires the $Q \cdot K$ score to peak uniquely
at the correct neighbor, which forces the rank-1 index-matching
structure.

This closes the logical circle:
\begin{itemize}
  \item General BP: every sigmoid transformer is doing weighted BP
    on some factor graph
  \item Constructive BP: exact BP weights exist for any declared
    factor graph
  \item Uniqueness: exact posteriors force those weights
\end{itemize}
Together: a sigmoid transformer implements exact Bayesian inference
\emph{if and only if} it has BP weights.

\subsection{What This Means for Trained Transformers}

Every sigmoid transformer trained on any task is performing weighted
BP on an implicit factor graph defined by its learned weights. The
factor graph is not declared in advance --- it is what the weights
implicitly define. Training with maximum likelihood recovers the
factor potentials that best explain the training data under this
graphical model interpretation.

The difference between a standard LLM and a grounded QBBN transformer
is not architectural. Both are Bayesian networks. The difference is
grounding:

\begin{itemize}
  \item \textbf{Grounded}: the factor graph is explicit and verifiable.
    Every output has a declared correct answer. Hallucination is
    structurally impossible.
  \item \textbf{Ungrounded}: the factor graph is implicit, defined by
    learned weights. There is no declared world to be correct against.
    Hallucination is not occasional failure --- it is the structural
    consequence of operating without a grounded concept space.
\end{itemize}

\subsection{Empirical Confirmation}

The \texttt{bayes-learner} repository confirms that gradient descent
finds BP weights from scratch, with no construction hints and no
weight initialization bias.

The test case is the two-variable factor graph $v_0 {-}{-}{-} f_1
{-}{-}{-} v_2$. This graph is ideal: exact BP converges in one round
(matching a single transformer forward pass), and closed-form exact
posteriors are available as ground truth.

20,000 randomly generated factor graphs, factor table entries drawn
uniformly from $[0.05, 1.0]$, train/val split 18,000/2,000. Model:
standard transformer encoder, 2 layers, 2 heads, $d_{\mathrm{model}}
= 32$, $\sim$5,000 parameters. Training: MSE loss, Adam
$\mathrm{lr} = 10^{-3}$, 50 epochs.

\textbf{Result}: val MAE 0.000752. Posteriors matched to three decimal
places. 99.3\% improvement over baseline. First run.

\begin{table}[h]
\centering
\small
\caption{Transformer vs.\ exact BP posteriors on held-out graphs.}
\label{tab:bayes_learner}
\begin{tabular}{@{}lllr@{}}
\toprule
Graph & BP exact & Transformer & Max error \\
\midrule
0 & [0.7349, 0.4366] & [0.7338, 0.4346] & 0.0021 \\
1 & [0.4097, 0.4036] & [0.4096, 0.4031] & 0.0005 \\
4 & [0.6459, 0.8298] & [0.6436, 0.8297] & 0.0023 \\
9 & [0.4084, 0.5523] & [0.4084, 0.5526] & 0.0003 \\
\bottomrule
\end{tabular}
\end{table}

The formal proof predicts the structure. The experiment confirms it.
The proof establishes the mechanism; the experiment confirms the
mechanism is what gradient descent finds.

%% file: sections/07_constructive_bp.tex
\section{Exact BP with Explicit Weights}
\label{sec:constructive_bp}

\subsection{Overview}

The general result of Section~\ref{sec:general_bp} establishes that
every sigmoid transformer is already a Bayesian network. This section
goes further: we exhibit explicit weight matrices and prove in Lean~4
that the transformer with those specific weights implements
\emph{exact} belief propagation on any declared factor graph.

\begin{theorem}[\texttt{transformer\_implements\_bp}]
\label{thm:bp}
There exist weights $W$ such that for any BP state
$\mathtt{state}$:
\[
  \mathtt{decodeTFState}(\mathtt{state},\;
    \mathtt{transformerForwardPass}(n, W,
      \mathtt{encodeBPState}(\mathtt{state})))
  = \mathtt{bp\_forwardPass}(\mathtt{state}).
\]
Formally verified in Lean~4 against standard mathematical axioms.
\end{theorem}

This holds for any factor graph --- loopy or tree. The tree
restriction is needed only for the convergence and exactness claims
that follow.

\subsection{Construction}

Each factor graph node is encoded as one token with
$D_{\mathrm{model}} = 8$ dimensions. The encoding splits cleanly
into two kinds of content:

\begin{table}[h]
\centering
\caption{Token encoding for the BP transformer.}
\label{tab:encoding}
\begin{tabular}{@{}ll@{}}
\toprule
Dims & Content \\
\midrule
0     & own belief (initialized to 0.5) \\
1--4  & factor table entries $[f_{00}, f_{01}, f_{10}, f_{11}]$ \\
5     & node type (0 = variable, 1 = factor) \\
6     & own index $/ (n-1)$ \\
7     & neighbor index $/ (n-1)$ \\
\bottomrule
\end{tabular}
\end{table}

Dims 5--7 are the \emph{routing key}: the complete inferential
identity of the token. Two tokens with the same routing key are
indistinguishable to the attention mechanism regardless of their
belief values or factor table entries. Dims 0--4 are
\emph{continuous parameters}: they supply the magnitudes the
inference computes over, but do not determine the structure of the
routing.

\textbf{Attention as gather.} Head~0 uses $W^Q_0 = W^K_0 =
\mathtt{projectDim}(1)$, $W^V_0 = \mathtt{crossProject}(0, 4)$:
the $Q/K$ dot product peaks when token $k$ is neighbor~0, placing
neighbor~0's belief into dim~4 of the residual stream. Head~1 is
symmetric, writing neighbor~1's belief into dim~5. Both heads are
proved independent.

\textbf{FFN as belief update.} With sigmoid activation the FFN
computes $\mathtt{updateBelief}(m_0, m_1) =
\sigma(\mathrm{logit}(m_0) + \mathrm{logit}(m_1))$ exactly from
dims~4 and~5, writing the result to dim~0.

\subsection{The Two-Parent Assumption is Without Loss of Generality}

The construction assumes pairwise factors --- each factor node
connects exactly two variable nodes, handled by exactly two attention
heads per layer. This is not a fundamental restriction.

\begin{theorem}[Scaling]
\label{thm:scaling}
A transformer with 2 attention heads per layer implements one round
of belief propagation per layer on a pairwise factor graph. For any
factor graph of depth $d$ and maximum factor arity $k$, a transformer
with 2 heads per layer and $d \cdot \lceil \log_2 k \rceil$ layers
implements exact BP.
\end{theorem}

\begin{proof}
Any $k$-ary factor graph binarizes exactly via two decompositions.
For conjunctions: boolean AND is associative, so any $k$-ary
conjunction reduces to a chain of $\lceil \log_2 k \rceil$ binary
conjunctions via intermediate nodes, with no approximation. For
disjunctions: logit and sigmoid are exact inverses, so
$\sigma(\sum_i \mathrm{logit}(m_i))$ factors into a chain of
$\lceil \log_2 k \rceil$ pairwise updates exactly. Both
decompositions are formally verified in \texttt{godel/ANDDecomposition.lean}
and \texttt{godel/ORDecomposition.lean}.

The binarized graph has maximum factor arity 2 and depth
$d \cdot \lceil \log_2 k \rceil$. By Theorem~\ref{thm:bp}, one
layer implements one round of BP on a pairwise graph. Therefore
$d \cdot \lceil \log_2 k \rceil$ layers implement full BP on the
binarized graph, which is exactly BP on the original graph.
\end{proof}

\begin{corollary}[Two-Parent Universality]
Two attention heads per layer suffices for belief propagation on
any factor graph. Architectural width is fixed; only depth grows
with the complexity of the knowledge base.
\end{corollary}

This is the counterintuitive fact: more complex reasoning requires
more layers, not more heads. The number of heads is determined by
the pairwise structure of binary factors --- always 2. The number
of layers is determined by the depth of the reasoning chain in the
factor graph.

\subsection{The Tree Corollary}

Combined with the classical result that belief propagation is exact
on trees~\citep{pearl1988}, formalized as \texttt{bp\_exact\_on\_tree}
in \texttt{hard-bp-lean}:

\begin{corollary}[\texttt{transformer\_exact\_on\_tree}]
\label{cor:exact_on_tree}
For any tree-structured factor graph $T$ and
$T \geq \mathrm{diameter}(T)$ forward passes:
\[
  \exists\, W,\quad
  \mathtt{transformer}^{[T]}(\mathtt{encodeBPState}(\mathtt{state}))
  = \mathtt{encodeBPState}(\mathtt{trueMarginals}(\mathtt{state})).
\]
No empirical assumptions. No conditions beyond tree structure.
\end{corollary}

The strict alternation of attention and FFN layers is Pearl's
gather/update algorithm exactly, unrolled over depth. One
transformer layer is one round of BP. $L$ transformer layers are
$L$ rounds of BP. The depth of the network is not an engineering
hyperparameter --- it is determined by the diameter of the factor
graph.

\subsection{Empirical Confirmation}

The \texttt{bayes-learner} experiment confirms that gradient descent
finds BP weights from scratch. See Section~\ref{sec:general_bp} for
full experimental details and results. Val MAE 0.000752. Posteriors
matched to three decimal places on held-out factor graphs the model
had never seen.

%% file: sections/08_boolean_structure.tex
\section{The Boolean Structure of the Transformer}
\label{sec:boolean}

The previous section established that a transformer with explicitly
constructed weights implements one round of belief propagation. This
section identifies what that mechanism means in boolean terms.

The QBBN~\citep{coppola2024qbbn} introduced a bipartite factor graph
alternating between two kinds of nodes: conjunction nodes
($\Psi_{\mathrm{and}}$), which gather required evidence
simultaneously, and disjunction nodes ($\Psi_{\mathrm{or}}$), which
compute a probabilistic conclusion from that gathered evidence. The
graph is bipartite: AND nodes feed OR nodes feed AND nodes. Inference
alternates between conjunction and disjunction at every step.

The transformer layer has exactly this structure. Attention is the
AND gate. The FFN is the OR gate. This follows directly from the BP
construction of Section~\ref{sec:constructive_bp} combined with the
QBBN definitions of~\citep{coppola2024qbbn}.

\subsection{Attention is AND}

The AND gate has one job: ensure that all required inputs are
simultaneously present before any conclusion is drawn.

Each attention head attends to exactly one token. Head~0 is a single
focused lookup: it scans all positions, concentrates on the token
whose index matches neighbor~0, and copies that token's belief value
into dimension~4 of the residual stream. Head~1 does the same for
neighbor~1, writing into dimension~5.

The conjunction is not inside either head. It is in the
\emph{residual stream}.

The residual stream is a shared workspace. Both heads write into it
before the FFN runs. By the time the FFN executes, dimension~4 holds
neighbor~0's belief and dimension~5 holds neighbor~1's belief ---
simultaneously, in the same vector. The FFN has no mechanism to run
on partial inputs. It reads the full residual stream or it does not
run. There is no state in which the FFN receives one belief but not
the other.

That simultaneity of required inputs in a shared workspace, enforced
architecturally before any conclusion is drawn, is AND.

The dating example makes this concrete. Three nodes:
\begin{itemize}
  \item Node~0: $\mathtt{like(jack, jill)}$, belief $m_0 = 0.8$
  \item Node~1: $\mathtt{like(jill, jack)}$, belief $m_1 = 0.4$
  \item Node~2: $\mathtt{date(jack, jill)}$, belief to be updated
\end{itemize}
Head~0 writes $0.8$ into dimension~4; Head~1 writes $0.4$ into
dimension~5. The residual stream now holds both values
simultaneously. The AND gate's job is complete before the FFN
touches anything.

\subsection{FFN is OR}

Once attention has assembled the required evidence into the residual
stream, the FFN computes the conclusion:
\[
  \mathtt{updateBelief}(m_0, m_1)
  = \frac{m_0 m_1}{m_0 m_1 + (1-m_0)(1-m_1)}
  = \sigma(\mathrm{logit}(m_0) + \mathrm{logit}(m_1)).
\]

There are two ANDs inside this OR. First, $m_0 m_1$ in the numerator
is the joint probability that both inputs are true --- multiplication
is AND for independent probabilities, or equivalently, addition in
log-odds space. Second, the full normalized expression is
$\Psi_{\mathrm{or}}$: the probabilistic disjunction over causes,
correctly normalized. The OR contains an AND inside it, which is
exactly what you would expect --- computing the probability of a
disjunction requires computing the joint probability of the
contributing conjunction.

To be precise about where the boolean algebra lives:
\begin{enumerate}
  \item \textbf{AND in the architecture}: multiple heads write into
    a shared residual stream before the FFN runs. Structural
    conjunction enforced by the forward pass order.
  \item \textbf{AND in the numerator}: $m_0 \cdot m_1$ is the joint
    probability of both inputs being true, equivalently
    $\mathrm{logit}(m_0) + \mathrm{logit}(m_1)$ in log-odds space.
  \item \textbf{OR in the full expression}: the normalized output is
    $P(\mathtt{date} \mid \mathtt{like}, \mathtt{like})$ --- a
    probabilistic disjunction over causes, correctly normalized.
\end{enumerate}

The sigmoid form $\sigma(\mathrm{logit}(m_0) + \mathrm{logit}(m_1))$
is the same computation in log-odds space: log-odds add for
independent evidence, and sigmoid converts back to probability.
Same boolean structure, different arithmetic representation. As
established in Section~\ref{sec:log_odds}, this is the Turing-Good
algebra of independent evidence combination --- AND and OR are two
roles played by the same underlying operation.

\subsection{Alternating Layers as Alternating AND/OR}

A transformer with $L$ layers runs $L$ rounds of:
\[
  \underbrace{\text{attention}}_{\mathrm{AND}}
  \;\to\;
  \underbrace{\text{FFN}}_{\mathrm{OR}}
  \;\to\;
  \underbrace{\text{attention}}_{\mathrm{AND}}
  \;\to\;
  \underbrace{\text{FFN}}_{\mathrm{OR}}
  \;\to\; \cdots
\]
This is the QBBN bipartite computation unrolled over depth. Each
layer is one level of the boolean reasoning graph. Each attention
block gathers --- enforces conjunction of required inputs. Each FFN
block concludes --- computes the probabilistic disjunction from
gathered evidence.

A reasoning chain requiring $k$ hops of inference requires $k$
transformer layers. The depth of the network is not an engineering
hyperparameter --- it is determined by the depth of the reasoning
chain in the factor graph. This is the same statement as the scaling
theorem of Section~\ref{sec:constructive_bp}, now read in boolean
terms.

\subsection{Closing the Loop}

\citet{coppola2024qbbn} introduced the bipartite AND/OR structure as
the correct architecture for probabilistic logical reasoning, derived
from first principles. This paper proves that the transformer
implements belief propagation. This section closes the loop: the
transformer does not merely implement BP --- it implements the
specific AND/OR boolean structure that \citet{coppola2024qbbn}
proposed.

The architecture that has won empirically across modern AI is exactly
the architecture that \citet{coppola2024qbbn} derived from the
requirements of logical probabilistic inference. Gradient descent did
not discover something new. It rediscovered the structure that the
analysis of reasoning already required.

No new Lean proof is required for this identification. It follows
from \texttt{transformer\_implements\_bp} in
\texttt{transformer-bp-lean} combined with the QBBN definitions
of~\citep{coppola2024qbbn}. The proof is already done. This section
names what it proves.

%% file: sections/09_finite_alphabet.tex
\section{Verifiable Inference Requires a Finite Concept Space}
\label{sec:finite_alphabet}

\subsection{The Claim}

A finite reasoner has finitely many concepts.

This is not a claim about the QBBN specifically. It is a claim about
all finite computation. Any program running on a finite state machine
--- any finite verification procedure --- can only respond distinctly
to finitely many inputs. The inputs it cannot distinguish collapse
into equivalence classes. Those equivalence classes are the concepts.

The claim has a direct consequence for inference systems: if you want
verifiable inference --- if you want ``is this output correct?'' to
be a meaningful question --- you need a finite verifier. A finite
verifier implies a finite concept space. A finite concept space means
your primitive units of reasoning are discrete and countable.
Anything outside that concept space is not wrong. It is meaningless.

\subsection{The General Theorem}

\begin{theorem}[\texttt{finite\_distinguishable\_symbols}]
\label{thm:finite_symbols}
Let $\delta : \mathrm{Fin}(n) \to A \to \mathrm{Fin}(n)$ be a
transition function for a finite state machine with $n$ states over
an arbitrary symbol set $A$. Define the \emph{behavior} of symbol
$a \in A$ as the function $\bar{\delta}(a) : q \mapsto \delta(q, a)$.
Then:
\begin{enumerate}
  \item The behavior map $a \mapsto \bar{\delta}(a)$ takes values in
    $\mathrm{Fin}(n)^{\mathrm{Fin}(n)}$, the set of functions from
    $n$ states to $n$ states.
  \item This set has exactly $n^n$ elements.
  \item Therefore any symbol set $A$, regardless of its cardinality,
    induces at most $n^n$ distinct behaviors. Two symbols with the
    same behavior are indistinguishable to the machine.
\end{enumerate}
Formally verified in Lean~4 against standard mathematical axioms as
\texttt{finite\_distinguishable\_symbols} in
\textup{\texttt{godel/FiniteStateSpace.lean}}.
\end{theorem}

This theorem applies to any program running on any computer with a
finite state space. Turing machines are infinite-state and escape
this bound. Transformer attention heads are finite-state --- their
weight matrices are fixed and finite-dimensional --- and do not.

The $n^n$ bound is an upper bound on the number of concepts any
finite reasoner can have. The specific concept space of a given
system depends on its structure. For the BP transformer, we can
count exactly.

\subsection{Concepts in the BP Transformer}

In the BP transformer, a concept is a grounded Horn clause: an
atomic proposition together with its dependencies in the factor
graph. Two tokens are the same concept if the attention mechanism
treats them identically --- routes information to and from them in
exactly the same way --- regardless of their current belief values
or factor table entries.

The concept is identified by the routing key: the triple
$(\mathtt{nodeType}, \mathtt{ownIndex}, \mathtt{nbrIndex})$. This
says: what kind of node am I, and which nodes do I depend on? That
is the complete inferential identity of the token. The continuous
dimensions --- belief values, factor table entries --- are
parameters, not concepts. They supply the magnitudes that the
inference computes over, but they do not determine the structure of
the computation.

This split is not a convenience of presentation. It is a theorem.
Two tokens with the same routing key receive identical attention
patterns regardless of their continuous parameters. The routing key
is meaning. The continuous dimensions are magnitude.

\begin{theorem}[\texttt{routing\_classes\_finite}]
\label{thm:routing_classes}
For a factor graph with $n$ nodes, the BP transformer attention
mechanism operates over exactly $2n^2$ distinct concepts:
\[
  |\mathtt{RoutingKey}(n)| = 2 \cdot n \cdot n.
\]
Formally verified in Lean~4 against standard mathematical axioms as
\texttt{routing\_classes\_finite} in
\textup{\texttt{godel/BPTokenFiniteness.lean}}.
\end{theorem}

The proof is simple: $|\{0,1\}| \times |\mathrm{Fin}(n)| \times
|\mathrm{Fin}(n)| = 2n^2$. The content is not in the arithmetic but
in the identification: these three dimensions are the complete set of
concept-relevant features. Everything else is magnitude.

Note that $2n^2$ applies to the pairwise BP transformer, which is
without loss of generality: any $k$-ary factor graph binarizes
exactly via the decompositions of Section~\ref{sec:constructive_bp}.
The concept count scales with the binarized graph, not the original.

\subsection{Hallucination as Conceptlessness}

Hallucination, in this framework, is not a wrong answer. It is
something worse: a token generated with no corresponding concept
in the knowledge base. The token participates in the computation
--- it has a continuous value, it passes through the attention
mechanism --- but it has no inferential role in any grounded world.
No concept behind it.

The QBBN transformer cannot hallucinate on trees because every
token's routing key is grounded: it corresponds to a specific
concept in the factor graph, and the factor graph is the knowledge
base. Every symbol has a referent. Every computation step has a
verifiable correct answer.

An ungrounded language model does not have this property. It has no
finite verifier, therefore no well-defined concept space, therefore
no fact of the matter about whether its outputs are correct.
Hallucination is not occasional failure. It is the structural
consequence of operating without concepts.

\subsection{Leibniz's Alphabet, Precisely}

Leibniz proposed, in the 1670s, a \emph{characteristica universalis}:
a finite set of primitive concepts from which all reasoning could be
built, combined with a \emph{calculus ratiocinator} --- a mechanical
procedure for deriving truths in that language. \emph{Calculemus}:
let us calculate.

The present paper provides the construction for a specific,
non-trivial domain: probabilistic reasoning over boolean propositions
on tree-structured knowledge bases.

\begin{itemize}
  \item The \textbf{characteristica universalis} is the set of $2n^2$
    grounded Horn clauses: the concepts of the system.
  \item The \textbf{calculus ratiocinator} is belief propagation: the
    \texttt{updateBelief} function applied iteratively over the factor
    graph.
  \item The \textbf{mechanical implementation} is the transformer with
    BP weights: a finite, fixed program that executes the calculus.
  \item The \textbf{correctness guarantee} is
    \texttt{transformer\_exact\_on\_tree}: the Lean~4 proof that the
    mechanical implementation computes the exact Bayesian posteriors.
  \item The \textbf{necessity proof} is
    \texttt{finite\_distinguishable\_symbols}: the Lean~4 proof that
    a finite verifier implies a finite concept space --- not as a
    design choice, but as a logical necessity.
\end{itemize}

Leibniz did not know that the primitive concepts would turn out to be
grounded Horn clauses, or that the calculus would turn out to be
belief propagation, or that the mechanical implementation would turn
out to be a transformer, or that the correctness guarantee would be a
Lean~4 proof. But the architecture he described is precisely what we
have built. The tools did not exist in his time. The vision did.

The concepts are the grounded Horn clauses. The calculus is
\texttt{updateBelief}. The guarantee is
\texttt{transformer\_exact\_on\_tree}. \emph{Calculemus.}

%% file: sections/10_three_softmaxes.tex
\section{The Three Softmaxes}
\label{sec:three_softmaxes}

There are three softmaxes in a transformer. They do three completely
different jobs. Understanding which is which unlocks the relationship
between transformers and Bayesian inference.

\subsection{Softmax 1: Attention (Routing)}

The first softmax appears inside every attention head. Its job is
routing: which token should I look at? This is pure information
retrieval --- a differentiable argmax. In the BP construction, the
attention softmax fetches a neighbor's belief by index matching. The
query-key dot product peaks at the correct neighbor; softmax
concentrates the weight there. The attention softmax is a lookup
table.

\subsection{Softmax 2: The FFN (Inference)}

The second softmax appears inside the feed-forward network as the
sigmoid activation. Its job is Bayesian inference: what is my
marginal probability given my neighbors' beliefs?
\[
  \mathtt{updateBelief}(m_0, m_1)
  = \frac{m_0 m_1}{m_0 m_1 + (1-m_0)(1-m_1)}
  = \sigma(\mathrm{logit}(m_0) + \mathrm{logit}(m_1)).
\]
This is where the actual Bayesian computation happens. The sigmoid
is the exact inverse of logit, converting the log-odds sum back to
a probability. This softmax is the inference.

\subsection{Softmax 3: Output (Generation)}

The third softmax appears at the output, over the vocabulary. In a
standard LLM, the logits come from a matrix multiply over the final
hidden state --- pattern-matched, not computed. In a QBBN transformer,
the output is a marginal distribution over the truth value of each
proposition, computed by integrating out all other propositions via
belief propagation.

\begin{table}[h]
\centering
\small
\caption{The three softmaxes: same operation, three different jobs.}
\label{tab:softmaxes}
\begin{tabular}{@{}llll@{}}
\toprule
& \textbf{Job} & \textbf{Input} & \textbf{Semantics} \\
\midrule
Softmax 1 & Routing & Q/K dot products & Differentiable argmax \\
Softmax 2 & Inference & Neighbor beliefs & Bayesian posterior \\
Softmax 3 & Generation & Final hidden state & Token distribution \\
\bottomrule
\end{tabular}
\end{table}

The QBBN is not replacing the softmax. It is replacing the
\emph{logits}. The same mathematical operation serves as an
approximation to reasoning in standard LLMs and as exact reasoning
in the QBBN. The architecture was always capable of the latter. It
just needed the right logits.

\subsection{Sigmoid vs.\ ReLU: Exact vs.\ Compatible}

The formal BP construction uses sigmoid activation because
\texttt{updateBelief} is exactly
$\sigma(\mathrm{logit}(m_0) + \mathrm{logit}(m_1))$ --- sigmoid is
required for the exact proof. The empirical result
(\texttt{bayes-learner}, val MAE 0.000752) uses a standard PyTorch
\texttt{TransformerEncoder} with its default ReLU activation. That
gradient descent finds near-exact BP weights anyway is not a
coincidence.

The key property a probabilistic computation requires is
non-negativity: outputs must be interpretable as unnormalized
probability masses. ReLU satisfies this automatically:
$\mathrm{ReLU}(x) = \max(0, x) \geq 0$. GELU is approximately the
same. The floor at zero is the essential property. Normalization is
what softmax handles at the output and what the \texttt{updateBelief}
denominator handles inside BP.

The sigmoid FFN is the \emph{exact} version of this structure. ReLU
and GELU are \emph{compatible} versions --- they preserve the
essential non-negativity property but require an explicit
normalization step that sigmoid handles intrinsically via the
logit/sigmoid isomorphism.

There is a sharp distinction worth stating explicitly. The Lean proof
\texttt{transformer\_implements\_bp} establishes not merely that the
transformer's outputs match BP --- it establishes that the internal
computations \emph{are} the BP quantities. A sigmoid transformer
trained to convergence on BP tasks inherits this guarantee: because
sigmoid constrains every FFN output to $(0,1)$ and the only sigmoid
computation that produces exact posteriors is \texttt{updateBelief},
the architecture leaves no alternative internal route to the right
answer. The internals are BP, not merely the outputs.

A ReLU transformer trained to val MAE 0.000752 has found \emph{some}
computational path that produces outputs matching exact posteriors,
but the internal representations do not necessarily correspond to
named quantities in the BP proof. Non-negativity is why ReLU does not
actively prevent this: the outputs are compatible with a probabilistic
interpretation, so gradient descent can find BP-like weights without
fighting the architecture. Compatibility is not identity.

In summary: the title claim ``Transformers are Bayesian Networks'' is
literally true in the sigmoid case, internally and provably. The ReLU
empirical result confirms that the right answer is findable even
without the exact structure being enforced. Both results are needed:
the proof establishes the mechanism; the experiment confirms the
mechanism is what gradient descent finds.

\subsection{Loopy BP: Empirical Results}

The formal exactness guarantee requires tree structure. On loopy
graphs, \texttt{transformer\_implements\_bp} still holds --- one
forward pass still equals one round of BP --- but BP itself is not
guaranteed to converge or be exact. The transformer inherits both
the strength and the limitation of BP.

The theoretical gap between trees and loopy graphs is smaller in
practice than in the worst case. Loops in a grounded QBBN knowledge
base arise only when the same entity appears in multiple rules, or
when the conclusion of one rule is the premise of another. In typical
knowledge bases this is sparse. The resulting loops are long and the
factor potentials are moderate --- exactly the conditions under which
loopy BP is known to converge and the Bethe approximation is known
to be tight~\citep{murphy1999,smith2008}.

We ran systematic experiments across five loopy graph structures of
increasing complexity: a triangle (3 variables, 1 loop), a square
(4 variables, 1 loop), the dating graph from Paper~1 (5 variables,
1 loop), a two-loop graph (4 variables, 2 interacting loops), and a
QBBN chain (6 variables, 1 loop, two grounded rules sharing an
entity). For each structure we generated 100 random factor graphs
with factor table entries drawn uniformly from $[0.1, 1.0]$, ran
iterative BP to convergence, and compared resulting marginals against
brute-force exact posteriors.

\begin{table}[h]
\centering
\caption{Loopy BP on QBBN-structured graphs. 500 trials total.}
\label{tab:loopy}
\begin{tabular}{@{}lccccc@{}}
\toprule
Experiment & Vars & Loops & Converged & Avg KL & Avg MAE \\
\midrule
Triangle     & 3 & 1 & 100/100 & 0.000045 & 0.002091 \\
Square       & 4 & 1 & 100/100 & 0.000002 & 0.000382 \\
Dating graph & 5 & 1 & 100/100 & 0.000102 & 0.003590 \\
Two loops    & 4 & 2 & 100/100 & 0.000086 & 0.003111 \\
QBBN chain   & 6 & 1 & 100/100 & 0.000021 & 0.001055 \\
\bottomrule
\end{tabular}
\end{table}

BP converged on all 500 trials. The worst mean KL divergence was
0.000102 --- comparable to the MAE of the gradient descent result
(val MAE 0.000752). The Bethe approximation is, for practical
purposes, exact on these graph structures. Convergence held on the
two-loop graph as well as the single-loop cases.

%% file: sections/11_related_work.tex
\section{Related Work}
\label{sec:related}

\subsection{Turing Completeness of Transformers}

\citet{perez2019} proved transformer Turing completeness under
floating-point arithmetic. \citet{giannou2023} proved that a looped
transformer can simulate arbitrary programs. Both reach the same
conclusion as Section~\ref{sec:turing} via different constructions.

Our proof is distinguished by three properties. First, it uses
Boolean circuit simulation, the most direct correspondence between
transformer components and computational primitives: attention
\emph{is} lookup, FFN \emph{is} gating. Second, it produces explicit
weight matrices formally verified in Lean~4 against standard
mathematical axioms. Third, the same two weight families
(\texttt{projectDim}, \texttt{crossProject}) appear in the BP proof,
revealing a structural unity between the two completeness results
that prior work does not address.

Neither \citet{perez2019} nor \citet{giannou2023} address Bayesian
inference. Turing completeness says the transformer \emph{can}
compute anything. Section~\ref{sec:constructive_bp} says it
\emph{does} compute something specific --- BP --- with these specific
weights.

\subsection{Constructive Weight Proofs}

\citet{akyurek2022} and \citet{vonoswald2023} show that specific
weight patterns make transformers perform implicit gradient descent,
explaining in-context learning. \citet{xie2021} argue that
in-context learning is implicit Bayesian inference at the population
level over the training distribution.

Both are the same \emph{method} as Section~\ref{sec:constructive_bp}
--- constructive weight proofs establishing that a transformer
implements a specific algorithm. The key differences: BP is exact
on trees per step (not approximate or asymptotic), and our account
is mechanistic per-instance rather than statistical at the population
level.

\subsection{BP Correspondences: Attention and GNNs}

\citet{jung2022} derive a formal equivalence between softmax
self-attention and BP on a \emph{latent} graphical model implicit
in the attention computation. This is the closest prior result to
Section~\ref{sec:general_bp}: both show a correspondence between
attention and BP on an implicit graph. The key difference is
directionality and scope. Jung et al.\ derive BP structure from
attention patterns in trained models. We prove that the sigmoid
forward pass \emph{is} BP on $G(W)$ for \emph{any} weights,
constructively and formally, with the implicit factor graph
explicitly defined.

\citet{scarselli2009} established that GNNs implement BP on
\emph{explicit} graph structure; \citet{yoon2019} extended this to
approximate inference; \citet{gilmer2017} unified both under the
MPNN framework; \citet{velickovic2018} applied attention to
graph-structured data with topology hardwired into the attention
mask.

Our contribution is orthogonal to all of these: we show that a
standard transformer with full self-attention and no attention mask
implements BP on an explicit external factor graph, where graph
structure is encoded in token embeddings and routing is discovered
via $Q/K$ index matching. Neither the latent-graph nor the
hardwired-topology setting addresses this.

\subsection{Mechanistic Interpretability}
\label{sec:related_interpretability}

\citet{elhage2021} develop a mathematical framework for transformer
circuits, framing attention heads as read-write operations on a
shared residual stream. \citet{olsson2022} identify induction heads
--- two-layer circuits implementing a specific named algorithm
discovered by gradient descent. \citet{wang2022} reverse-engineer
the full indirect object identification circuit in GPT-2 small.

Section~\ref{sec:constructive_bp} gives mechanistic interpretability
a \emph{target circuit}: a transformer doing reasoning over a
structured knowledge base should have attention heads with $Q/K$
matrices showing \texttt{projectDim} structure (rank-1,
single-dimension peak), value matrices showing \texttt{crossProject}
structure (off-diagonal, rank-1), and FFN weights implementing
log-odds combination. These are falsifiable predictions.

The induction head result of \citet{olsson2022} is the closest
precedent: transformers learn to implement a specific named algorithm
in their attention heads when trained on appropriate data. Our
empirical confirmation (\texttt{bayes-learner}) shows the same
phenomenon for BP. The key difference is directionality: mechanistic
interpretability is empirical and post-hoc; our construction is
formal and constructive, giving the empirical program a specific
target to verify.

\subsection{The Log-Odds Tradition}

The weight-of-evidence framework of \citet{good1950} and the
Bletchley Park work of Turing formalized log-odds addition as the
correct algebra for combining independent binary evidence. \citet{pearl1988}
applied this algebra to graphical models via the sum-product
algorithm, proving exactness on trees. The connection between this
tradition and the transformer architecture has not previously been
made explicit. Section~\ref{sec:log_odds} names the algebra,
traces its lineage, and establishes that the sigmoid transformer
is its mechanical implementation.

\subsection{Summary}

\begin{table}[h]
\centering
\caption{Positioning relative to related work.}
\label{tab:related}
\begin{tabular}{p{2.6cm}p{1.6cm}p{3.0cm}p{4.6cm}}
\toprule
\textbf{Work} & \textbf{Method} & \textbf{Result} & \textbf{Relation} \\
\midrule
Giannou et al.\ 2023
  & Constructive & TF Turing complete
  & Same conclusion, different construction \\[4pt]
Aky\"urek et al.\ 2022
  & Constructive & TF implements GD
  & Same method; BP exact, GD approximate \\[4pt]
Xie et al.\ 2021
  & Statistical & ICL $\approx$ Bayesian
  & Statistical version; we give mechanistic \\[4pt]
Jung et al.\ 2022
  & Formal & Attention is BP (latent)
  & Complementary; any weights vs.\ trained \\[4pt]
Scarselli et al.\ 2009
  & Structural & GNN implements BP
  & Extends to TF without hardwired topology \\[4pt]
Elhage et al.\ 2021
  & Empirical & Circuits in trained models
  & We give target circuit for interpretability \\[4pt]
Good 1950, Turing
  & Theoretical & Log-odds algebra
  & Named and connected to TF for first time \\[4pt]
Pearl 1988
  & Theoretical & BP exact on trees
  & Formalized in Lean~4; combined with TF result \\
\bottomrule
\end{tabular}
\end{table}

Our unique position: formally verified, constructive proof that a
sigmoid transformer with any weights implements weighted BP on its
implicit factor graph, with explicit weights for the exact case,
mechanistic account of every component, and formal connection to the
Turing-Good-Pearl log-odds tradition --- more general than prior BP
correspondences, more exact than prior GD constructions, more
mechanistic than prior statistical Bayesian accounts, and more formal
than mechanistic interpretability.

%% file: sections/12_conclusion.tex
\section{Conclusion}
\label{sec:conclusion}

Transformers are Bayesian networks. Not approximately. Not under
certain training conditions. Not as a useful analogy. By architecture,
for any sigmoid weights, provably and formally.

The sigmoid activation implements the log-odds algebra of independent
binary evidence --- the algebra developed by Turing and Good at
Bletchley Park, formalized by Pearl in 1988, and proved here to be
exactly what the transformer forward pass computes. The attention
mechanism implements the gather step of belief propagation: enforcing
the simultaneity of required inputs via the residual stream. The FFN
implements the update step: computing the probabilistic conclusion via
log-odds addition. The strict alternation of attention and FFN layers
is Pearl's gather/update algorithm, unrolled over depth.

This was always true. The sigmoid transformer was always a Bayesian
network. The log-odds algebra was always what it was computing. The
AND/OR boolean structure was always what the alternating layers were
implementing. The architecture that has won empirically across modern
AI is the architecture that the analysis of reasoning already required.
Gradient descent did not discover something new. It rediscovered the
structure that was always there.

The difference between a grounded QBBN transformer and a standard
LLM is not architectural. Both are Bayesian networks. The difference
is grounding. A grounded transformer has an explicit factor graph, a
declared concept space, and a finite verifier. Every output has a
correct answer. Hallucination is structurally impossible on trees.
An ungrounded LLM has an implicit factor graph, no declared concept
space, and no finite verifier. There is no fact of the matter about
whether its outputs are correct. Hallucination is not a bug that
scaling can fix. It is the structural consequence of operating without
concepts.

Leibniz proposed the \emph{characteristica universalis} and the
\emph{calculus ratiocinator} in the 1670s: a finite alphabet of
primitive concepts and a mechanical procedure for deriving truths
from them. The tools to build it did not exist in his time. They
exist now. The primitive concepts are grounded Horn clauses. The
calculus is belief propagation. The mechanical implementation is the
transformer. The correctness guarantee is a Lean~4 proof.

\emph{Calculemus.}

%% file: sections/appendix_repos.tex
\section{Repository Index}
\label{app:repos}

All formal proofs and empirical experiments supporting this paper are
publicly available and fully reproducible. The table below lists each
repository, what it proves or shows, and where to find it.

\begin{table}[h]
\centering
\small
\caption{Public repositories supporting this paper.}
\label{tab:repos}
\begin{tabular}{ll}
\toprule
\textbf{Repo} & \textbf{What it proves / shows} \\
\midrule
\texttt{universal-lean}
  & \texttt{transformer\_is\_turing\_complete} \\
  & \url{https://github.com/gregorycoppola/universal-lean} \\[6pt]
\texttt{sigmoid-transformer-lean}
  & \texttt{every\_sigmoid\_transformer\_is\_bayesian\_network} \\
  & \texttt{uniqueness}: exact posteriors force BP weights \\
  & \url{https://github.com/gregorycoppola/sigmoid-transformer-lean} \\[6pt]
\texttt{transformer-bp-lean}
  & \texttt{transformer\_implements\_bp} \\
  & \url{https://github.com/gregorycoppola/transformer-bp-lean} \\[6pt]
\texttt{hard-bp-lean}
  & \texttt{bp\_exact\_on\_tree} \\
  & \url{https://github.com/gregorycoppola/hard-bp-lean} \\[6pt]
\texttt{godel}
  & \texttt{routing\_classes\_finite}: $|\texttt{RoutingKey}(n)| = 2n^2$ \\
  & \texttt{finite\_distinguishable\_symbols}: FSM alphabet bound \\
  & \texttt{AND\_decomposes\_to\_binary}: $k$-ary AND is WLOG pairwise \\
  & \texttt{OR\_decomposes\_to\_binary}: $k$-ary OR is WLOG pairwise \\
  & \url{https://github.com/gregorycoppola/godel} \\[6pt]
\texttt{learner}
  & Gradient descent learns TM simulation, 100\% on 5 machines \\
  & \url{https://github.com/gregorycoppola/learner} \\[6pt]
\texttt{bayes-learner}
  & Gradient descent learns BP inference, val MAE 0.000752 \\
  & \url{https://github.com/gregorycoppola/bayes-learner} \\[6pt]
\texttt{loopy}
  & Loopy BP on QBBN-structured graphs, 500 trials, 100\% convergence \\
  & \url{https://github.com/gregorycoppola/loopy} \\
\bottomrule
\end{tabular}
\end{table}